\documentclass{article}

\usepackage[utf8]{inputenc} 
\usepackage[T1]{fontenc}    

\usepackage{natbib}
\usepackage{url}            
\usepackage{booktabs}       
\usepackage{amsfonts}       
\usepackage{nicefrac}       
\usepackage{microtype}      
\usepackage{xcolor}         
\usepackage{enumitem}
\usepackage[most]{tcolorbox}
\usepackage{graphicx}
\usepackage{subcaption}
\usepackage{minted}
\usepackage{amsthm}
\usepackage{mathtools}
\usepackage{thmtools}
\usepackage{booktabs}
\usepackage{xspace}
\usepackage{comment}
\usepackage[pdftex,colorlinks,linkcolor=blue,citecolor=blue,filecolor=blue,urlcolor=blue]{hyperref}
\usepackage[capitalise,noabbrev]{cleveref}
\usepackage[letterpaper,hmargin=1.0in,vmargin=1.0in]{geometry}
\usepackage{svg}
\graphicspath{{pics/}}

\newboolean{isInternal}
\setboolean{isInternal}{false}

\ifthenelse{\boolean{isInternal}}{\includecomment{internal}}{\excludecomment{internal}}

\ifthenelse{\boolean{isInternal}}{

  \newcommand{\internalComment}[1]{\textbf{\color{red}[}{\footnotesize#1}\textbf{\color{red}]}}
}{
  \newcommand{\internalComment}[1]{}
}

\DeclareMathOperator*{\argmax}{arg\,max}

\declaretheoremstyle[
headfont=\bfseries,
bodyfont=\slshape,
sibling=subsection,
spaceabove=\medskipamount,
spacebelow=\medskipamount,
]{plain}

\declaretheorem[style=plain,name=Example]{example}

\declaretheorem[style=plain,name=Proposition]{proposition}

\declaretheoremstyle[
headfont=\bfseries,
sibling=subsection,
spaceabove=\medskipamount,
spacebelow=\medskipamount,
]{definition}

\declaretheorem[style=definition,name=Definition]{definition}

\declaretheoremstyle[
headfont=\itshape,
sibling=subsection,
spaceabove=\medskipamount,
spacebelow=\medskipamount,
]{remark}

\definecolor{mplblue}{HTML}{1f77b4}
\definecolor{mplorange}{HTML}{ff7f0e}
\definecolor{mplgreen}{HTML}{2ca02c}
\definecolor{mplred}{HTML}{d62728}
\definecolor{mplpurple}{HTML}{9467bd}
\definecolor{mplbrown}{HTML}{8c564b}
\definecolor{mplpink}{HTML}{e377c2}
\definecolor{mplgray}{HTML}{7f7f7f}
\definecolor{mplolive}{HTML}{bcbd22}
\definecolor{mplcyan}{HTML}{17becf}

\newtheorem*{theorem*}{Theorem}

\newcommand{\agentk}{\text{Agent-$\nicefrac{1}{3}\times 3$\xspace}}

\newcommand{\kshot}{$k$-shot\xspace}

\definecolor{roundnotecolor}{HTML}{19519a} 
\newtcolorbox{roundnote}[1][]{
  enhanced,
  breakable,
  colback=roundnotecolor!3,
  colframe=roundnotecolor,
  boxrule=0.5pt,
  arc=3mm,
  left=6pt,right=6pt,top=6pt,bottom=6pt,
  before upper=\setlength{\parskip}{\medskipamount},
  #1
}

\newcommand\MTkillspecial[1]{
  \bgroup
  \catcode`\&=9
  \let\\\relax%
  \scantokens{#1}%
  \egroup
}
\newcommand{\DeclareCustomDelim}[3]{
  \DeclarePairedDelimiter{#1}{#2}{#3}
  \reDeclarePairedDelimiterInnerWrapper{#1}{star}{
    \mathopen{##1\vphantom{\MTkillspecial{##2}}\kern-\nulldelimiterspace\right.}
  ##2
  \mathclose{\left.\kern-\nulldelimiterspace\vphantom{\MTkillspecial{##2}}##3}
  }
}

\DeclareCustomDelim{\prn}{\lparen}{\rparen}
\DeclareCustomDelim{\crl}{\{}{\}}
\DeclareCustomDelim{\brk}{[}{]}
\DeclareCustomDelim{\norm}{\|}{\|}
\DeclareCustomDelim{\abs}{|}{|}

\DeclarePairedDelimiterXPP\Prob[1]{\Problet}\{\}{}{
  
  #1}
\DeclarePairedDelimiterXPP\Expect[1]{\Expectlet}[]{}{
  
  #1}

\title{
When Independent Sampling Outperforms Agentic Reasoning
}

\author{
  Yihe Dong\thanks{Equal contribution} \\
  Princeton University\\
  \texttt{ydong@princeton.edu} \\
  \and
  Boris Shigida\footnotemark[1] \\
  Princeton University\\
  \texttt{bs1624@princeton.edu} \\
}

\date{}

\NewDocumentCommand{\py}{+v}{}

\newsavebox{\pycodebox}      

\NewDocumentEnvironment{pycode}{}{%
  \VerbatimEnvironment%
  \begin{lrbox}{\pycodebox}%
  \begin{minipage}{\linewidth}%
  \begin{Verbatim}%
}{
  \end{Verbatim}%
  \end{minipage}%
  \end{lrbox}%
}

\NewDocumentEnvironment{pycontext}{} {
  \VerbatimEnvironment
  \begin{lrbox}{\pycodebox}%
  \begin{minipage}{\linewidth}%
  \begin{Verbatim}
}{
  \end{Verbatim}
  \end{minipage}
  \end{lrbox}
}

\NewDocumentEnvironment{nothing}{}{}{}

\begin{pycontext}
import os
import sys
os.chdir("../code")
sys.path.insert(0, os.getcwd())
from report import *
\end{pycontext}

\begin{document}

\maketitle

\begin{abstract}

We study how to allocate inference-time compute for competitive programming under fixed budgets.
Evaluating 216 Codeforces problems across Divisions~1--3, we compare agent-based reasoning with repeated independent sampling (\kshot) as a function of both cost and number of model calls.
Across models and difficulty levels, \kshot consistently achieves a better accuracy-cost and accuracy-query tradeoff.
This gap persists despite prompt caching in agent frameworks, indicating lower per-call effectiveness.
Our results show that, for self-contained algorithmic tasks, independent exploration can outperform deeper agentic reasoning under realistic resource constraints.
We also provide a budget-allocation analysis when the inference budget is fixed, and prove that a cost-optimal solver minimizes the principled metric log failure likelihood per dollar.%
\footnote{Our code is publicly available at \url{https://github.com/princeton-pli/competitive-programming-agents}.}
\end{abstract}




\section{Introduction}

Large language models (LLMs) have recently demonstrated strong capabilities on algorithmic reasoning and competitive programming tasks, building on the Transformer architecture \citep{vaswani2017attention} and benefiting from increased test-time computation. Prior work has shown that increasing inference-time compute through techniques like repeated sampling, self-consistency, or majority voting can substantially improve reasoning accuracy without additional training \citep{debateOrVote}. In parallel, \emph{agent-based} approaches have emerged as a powerful paradigm, enabling models to interleave reasoning with actions such as code execution, debugging, and environment interaction \citep{yao2023react,li2023theoryofmind}. These methods have achieved strong results in complex decision-making and software engineering settings, including state-of-the-art performance on SWE-bench \citep{jimenez2024swebench,NEURIPS2024_5a7c9475}.

In the context of competitive programming, recent work has explored both of these aprooaches.
\Citet{zheng2025livecodebenchproolympiadmedalists} evaluate large reasoning models on challenging competitive programming problems and observe a substantial gap between benchmark performance and results reported in official model releases.
They argue that this gap can be partially closed by granting models terminal access, enabling behaviors such as writing brute-force solvers to validate optimized implementations.
OpenAI et al.~\citep{openai2025competitiveprogramminglargereasoning}
implement this along with additional methods of improving performance, such as targeted fine-tuning and guidance, giving the model an opportunity to develop complex test-time strategies, etc.
While these results suggest that agentic interaction and feedback can, in principle, improve competitive programming performance, such approaches incur significant inference-time cost, and potentially require complex custom environments. Existing evaluations typically assume effectively unbounded resources or report performance at a single operating point without focusing on the cost.

In practice, real-world deployments of LLMs operate under strict inference budgets, whether measured in monetary cost, number of model calls, or latency. Under such constraints, the decision of primary importance is how best to allocate limited resources: should one invest them into a single agent trajectory with rich interaction and feedback, or distribute them across multiple independent attempts that favor exploration? More broadly, existing work lacks a systematic understanding of the \emph{accuracy--cost tradeoff} induced by different inference strategies.

\begin{figure}[htb!]
\centering
\includegraphics[width=0.450\textwidth]{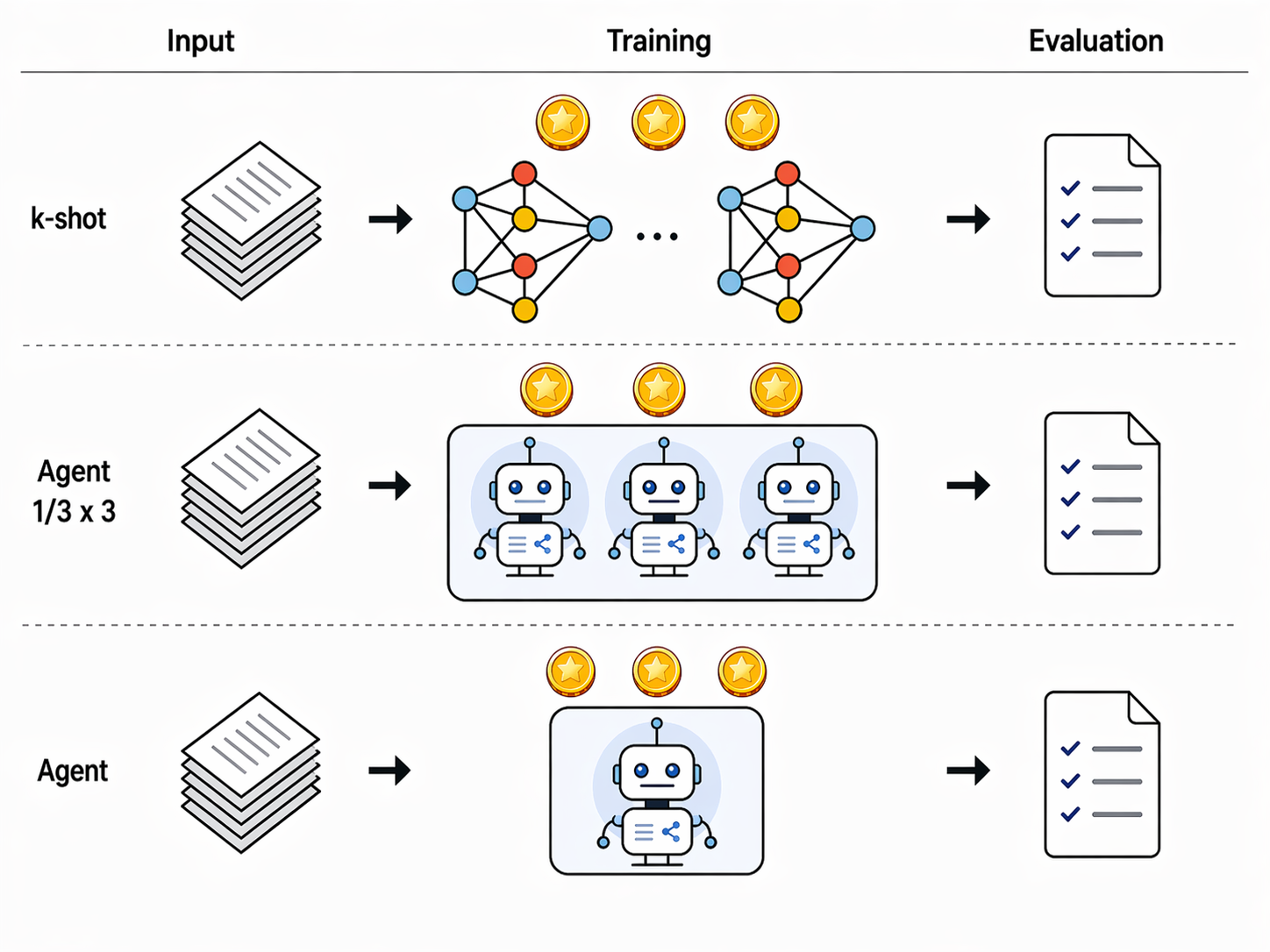}
\hfill
\includegraphics[width=0.265\textwidth]{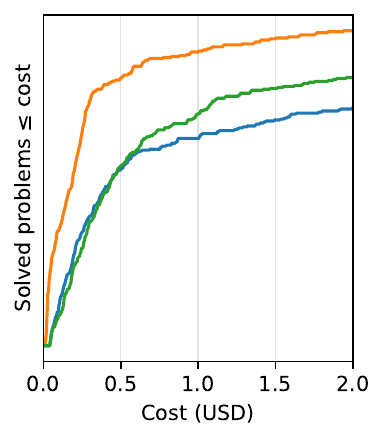}
\hfill
\includegraphics[width=0.265\textwidth]{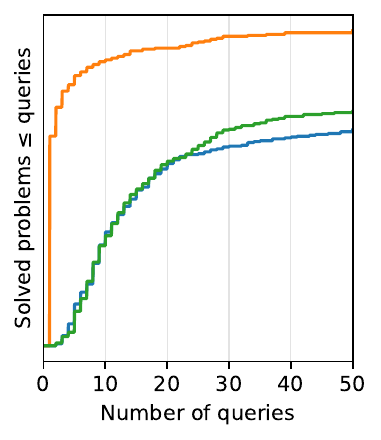}
\caption{
\textbf{Left}: overview of our three evaluation settings.
From top to bottom: \textit{\kshot}: each problem is attempted $k$ times independently by one API call; \textit{\agentk} budget-partitioned agents where three independent SWE-agent instances are given $c / 3$ dollars to solve the problem; \textit{Agent}: one full SWE-agent run given $c$ dollars (where $c$ is the budget).
\textbf{Center and right}:
averaged cumulative solved problems (across all models and divisions) versus inference cost and number of queries: {\color{mplorange}\kshot inference} outperforms {\color{mplgreen}budget-partitioned agents} and {\color{mplblue}one full SWE-agent} by both metrics.}
\label{fig:overview}
\end{figure}

In this work, we address this gap by studying a simple but fundamental question: \emph{given a fixed inference budget, which strategy yields the best performance on competitive programming problems?} Rather than focusing on a single accuracy number, we characterize the \emph{entire trajectory of performance as a function of cost}, providing a fine-grained comparison across inference regimes. Competitive programming is a particularly clean testbed for testing this accuracy-cost tradeoff, as each problem is self-contained, automatically judged, and has well-defined environment navigation.
We evaluate three representative strategies:
\begin{enumerate}[label=(\roman*)]
\item \textbf{$k$-shot}, which repeatedly samples independent solutions;
\item \textbf{Agent}, which allows iterative reasoning and terminal interaction within a fixed budget;
\item \textbf{\agentk, or budget-partitioned agents}, which split the same total budget across multiple independent agent runs.
\end{enumerate}
All methods are evaluated under matched cost or query constraints on a diverse set of Codeforces\footnote{\url{https://codeforces.com/}} problems spanning Divisions~1--3. The problems were chosen to cover a range of difficulty levels, and their publication dates post-date the release dates for the models tested to avoid data contamination.

Our results reveal a surprising and consistent pattern.
\begin{roundnote}
Across model families and difficulty levels, $k$-shot inference significantly outperforms agent-based methods when performance is normalized by either cost or number of model calls (\cref{fig:overview}).
\end{roundnote}
This gap persists even though agent frameworks cache context and amortize prompt costs to reduce the disadvantage of agents spanning from longer contexts or higher per-step overhead. While agent-based methods have demonstrated clear advantages in settings such as software engineering benchmarks (where navigation, file manipulation, and iterative testing are unavoidable), their overhead appears poorly matched to the structure of competitive programming tasks under tight resource constraints.

To better understand the failures of agents, we analyze agent trajectories
using LLM-as-a-judge
\citep{gu2024survey}
and identify recurring failure modes, including unproductive iterative loops, ineffective debugging strategies, and inability to find the key algorithm. Analogous inefficiencies have been observed in prior studies of multi-agent debate and collaboration, where much of the observed performance gain can be attributed to aggregation of independent trials rather than multi-agent interaction
\citep{debateOrVote,madSparseTopology,MAD}. Indeed, independent $k$-shot attempts naturally emphasize exploration, allowing rare but correct solution paths to be discovered early at relatively low cost.

To place these observations in a broader context, we provide a theoretical discussion of how to allocate resources to success-or-fail solvers under a fixed budget. We model this problem as an integer program, solve a linear relaxation and arrive at the following rule of thumb.
\begin{roundnote}
  The cost-optimal solver with success probability $p$ and cost $c$ dollars is the one that minimizes the log failure likelihood per dollar
\begin{equation}\label{eq:GNnqtn}
\text{minimize} \quad \frac{\ln (1 - p)}{c}.
\end{equation}
When the goal is to maximize the final probability of solving the problem under fixed (large enough) budget $C$, one should run the cost-optimal solver many times independently (until success is reached) rather than use any other solver.
\end{roundnote}

The problem of maximizing the system reliability under budget constraints is well-studied \citep{guan2025review}.
Although the proof method is not new,
we are not aware of literature applying it in similar contexts and analyzing the resulting scalar metric of accuracy-cost tradeoff.
In \cref{sec:related-work}, we extensively discuss works that studied other intuitively motivated metrics
or different settings.
This provides a new perspective on the scaling science of agent systems, introducing a
simple and mathematically correct method of comparing agent systems based on their accuracy and cost rather than accuracy alone.
As an extreme case, a task can be agentic \citep{kim2025towards} in the sense that adaptive interaction with the environment benefits the expected success rate, but a zero-interaction approach can still be more cost-efficient.

In particular, if one API call is more cost-efficient in the sense of \eqref{eq:GNnqtn} than an agent run, one should repeat API calls independently ($k$-shot). Further, stateful searches like agent runs have hyperparameters (such as the model softmax temperature at inference, maximal cost per attempt), and the best ones in the sense of \eqref{eq:GNnqtn} should be chosen to minimize the total costs. This goes against the naive intuition that longer searches are always better (in fact, some finite search length can be the best).

To apply this framework, we estimate success and cost scaling trends for independent sampling and agentic runs with different cost limits (\cref{fig:scaling}). Consistent with the evaluations discussed above and \cref{fig:overview}, an API call is more cost-efficient in our setting. Further, the cost-optimal cost limit per attempt for an agent is finite in our experiment. The details on the scaling trends are discussed in \cref{sec:broader-discussion}.

Our main contributions are:

\begin{itemize}
\item \textbf{Systematic comparison of inference regimes.}
  We empirically compare $k$-shot inference, agent-based inference, and budget-partitioned agents under matched cost and query constraints across multiple model families and difficulty levels.\\ \emph{Key finding:} $k$-shot inference consistently achieves a superior accuracy-cost tradeoff, often matching or exceeding agent performance at a fraction of the cost (\cref{fig:cum-solved-vs-cost-o3,fig:cum-solved-vs-queries-o3}).

\item \textbf{Failure mode analysis.} Through trajectory-level analysis, we identify common agent weaknesses, such as inefficient iteration and ineffective debugging, that explain their poor cost efficiency in this setting.

\item \textbf{Principled cost-accuracy tradeoff metric.}
We study competitive programming performance as a continuous function of inference cost, rather than at a single operating point.
Accordingly, we introduce a simple metric (log-failure likelihood per dollar)
that captures the accuracy vs. cost tradeoff in a theoretically principled way rather than just being intuitively plausible (\cref{sec:broader-discussion}).

\item \textbf{Empirical scaling trends.}
We rigorously estimate scaling trends (\cref{fig:scaling}) of $k$ independent attempts and agents with different cost budgets in our competitive programming setting (controlling for uncertainty by enforcing an upper bound on the length of exact confidence intervals).

\end{itemize}



\section{Methods}

\paragraph{Problem set.}
We evaluate on a representative suite of Codeforces problems drawn from 10 contests spanning Divisions~1,~2, and~3.
Each contest contains 6--8 problems ordered (approximately) by difficulty, yielding a total of 216 problems across the three divisions.
This design provides broad coverage over problem difficulty while keeping the evaluation tractable and reproducible.

\paragraph{Evaluation objective and budgets.}
For each problem, we measure the cumulative number of solved problems as a function of inference resources, under two complementary budget regimes:
(i) a \emph{monetary} budget, and (ii) a \emph{query-count} budget.
Concretely, we cap per-problem spending at $c_{\max}$ USD and separately cap the number of model calls at $k_{\max}$.
Unless otherwise noted, we set $c_{\max}=2.0$ (USD per problem) and $k_{\max}=50$ (model calls per problem), matching a realistic deployment scenario with strict per-instance constraints.
All experiments use temperature 0.5 when available (and the provider default otherwise).
Prompts and system instructions are provided in \cref{sec:prompts}.

\paragraph{Inference settings.}
We compare three inference strategies that differ in how they allocate a fixed budget between \emph{exploration} (independent attempts) and \emph{exploitation} (iterative refinement with tool feedback).
In all settings, a problem is marked as solved if the produced solution is accepted by the Codeforces judge under the standard evaluation protocol.

\paragraph{$k$-shot inference (independent attempts).}
This is the simplest baseline.
Each attempt corresponds to a single API call that produces a complete solution (including explanation and final code).
Under the monetary budget regime, we repeatedly sample independent attempts until the cumulative cost reaches $c_{\max}$, which yields approximately
$\left\lceil c_{\max} / \mathrm{cost}(\text{one call}) \right\rceil$ attempts.
Under the query-count regime, we run exactly $k_{\max}$ attempts (excluding transient API errors and automatic retries).
Among the sampled attempts, we select the best candidate according to the evaluation protocol (i.e., the first accepted submission, if any).

\paragraph{SWE-agent (single trajectory with tools).}
To instantiate an agentic programming workflow, we use SWE-agent \citep{NEURIPS2024_5a7c9475} as an off-the-shelf agent framework.
For each problem, we run a single SWE-agent trajectory with an end-to-end budget of $c_{\max}$.
The agent is granted access to a terminal for iterative coding, execution, and debugging, and may invoke the model multiple times within the budget.
Prompt caching is enabled throughout to amortize repeated context across successive calls.

\paragraph{Budget-partitioned SWE-agent ($\nicefrac{1}{3}\times 3$).}
To disentangle the effect of \emph{multiple independent runs} from \emph{within-trajectory refinement}, we also evaluate a partitioned variant of the agent.
Here we run SWE-agent three times independently on the same problem, each with budget $c_{\max}/3$.
We then aggregate across runs in the same way as in $k$-shot inference (a problem is solved if any run produces an accepted solution).
This setting tests whether distributing a fixed budget across multiple agent initializations can recover the benefits of exploration while retaining tool use. We created a SWE-agent fork with custom scaffolding and tool calling specific for competitive programming.

\paragraph{Implementation details and reporting.}
For cost-based comparisons, we record the cumulative monetary cost incurred as the evaluation progresses, and report the number of problems solved at each cost threshold.
For query-based comparisons, we analogously report the number of problems solved as a function of the number of model calls.
These two views allow us to separate \emph{per-call effectiveness} from \emph{per-dollar efficiency}, which is critical when comparing methods that differ substantially in orchestration overhead and tool usage patterns. Our evaluation of solutions uses a dedicated new  Codeforces API (currently private) designed for AI research. 


\begin{pycontext}
import os
import sys
os.chdir("../code")
sys.path.insert(0, os.getcwd())
from report import *
\end{pycontext}

\section{Results and Interpretation}
\label{sec:results}

\begin{figure}[htb!]
  \centering

  \includegraphics[width=0.900\textwidth]{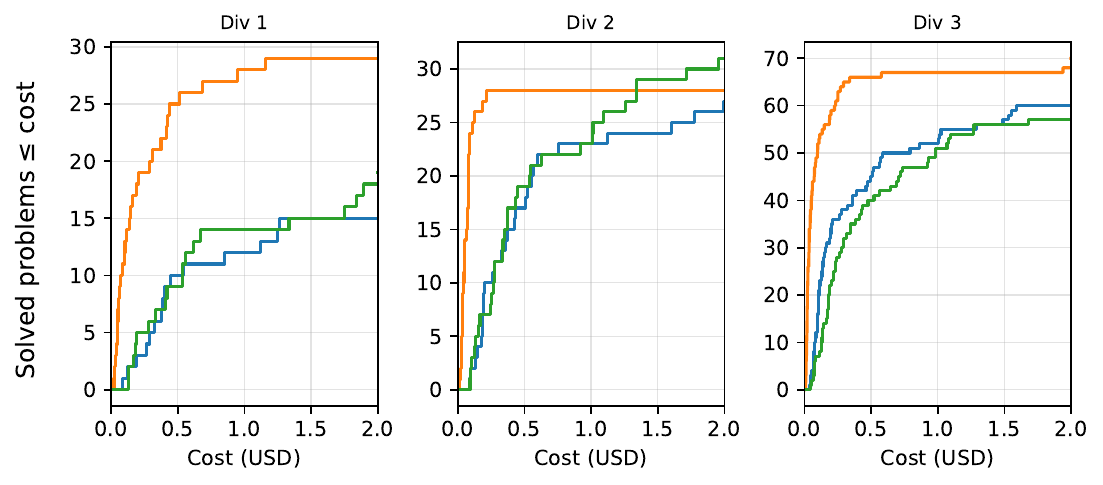}

  \caption{Cumulative solved problems versus inference cost with OpenAI o3 for: {\color{mplblue}agent}, {\color{mplorange}$k$-shot}, {\color{mplgreen}budget-partitioned agent ($\nicefrac{1}{3}\times 3$)}.  \cref{sec:additional-results} contains results for additional models.}
  \label{fig:cum-solved-vs-cost-o3}
\end{figure}

Figure~\ref{fig:cum-solved-vs-cost} reports the cumulative number of solved problems as a function of inference cost.
Across all evaluated model families and Codeforces difficulty levels (Divisions~1--3), \kshot consistently dominates agent-based approaches when normalized by monetary cost.
In particular, for any fixed budget threshold, \kshot solves strictly more problems than both a single SWE-agent run and the budget-partitioned SWE-agent ($\nicefrac{1}{3} \times 3$) variant.
This trend holds uniformly across easy, medium, and hard problems, indicating that the advantage of \kshot is not confined to a narrow difficulty regime.

\begin{figure}[t]
    \centering    \includegraphics[width=0.90\textwidth]{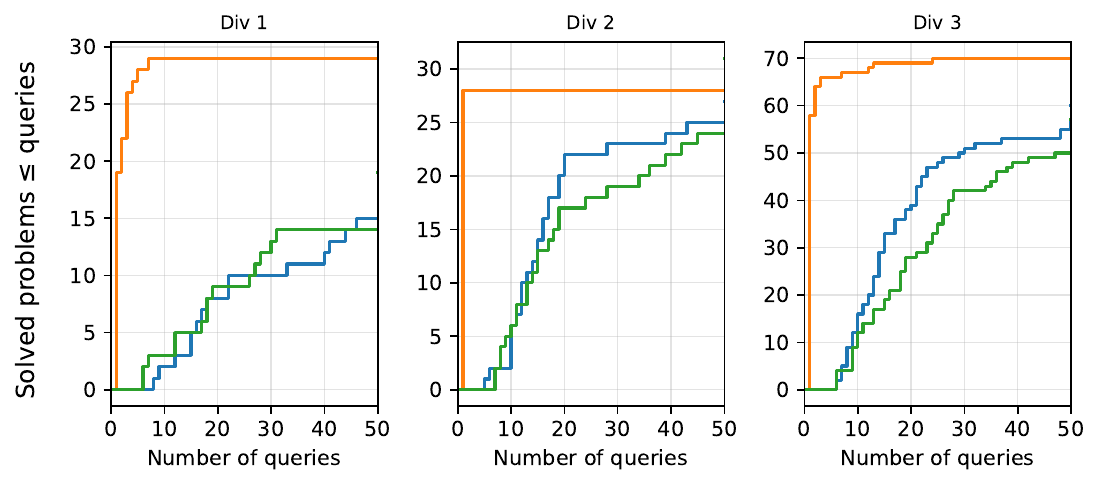}
    \caption{Cumulative solved problems versus number of queries with OpenAI o3 for: {\color{mplblue}agent}, {\color{mplorange}$k$-shot}, {\color{mplgreen}budget-partitioned agent ($\nicefrac{1}{3}\times 3$)}.  \cref{sec:additional-results} contains additional models.}
    \label{fig:cum-solved-vs-queries-o3}
\end{figure}

Importantly, the observed gap is not an artifact of agents issuing a larger number of model calls.
Figure~\ref{fig:cum-solved-vs-queries} plots performance as a function of the number of model queries, independent of their monetary cost.
Even under this normalization, agent-based methods substantially lag behind \kshot, demonstrating that their inferior cost efficiency cannot be explained solely by higher per-call overhead.
Instead, agents are less effective \emph{per model invocation}, despite benefiting from iterative refinement and tool access.

Notably, this result persists even though SWE-agent employs prompt caching, so successive model calls do not incur costs that scale linearly with the growing context length \citep{NEURIPS2024_5a7c9475}.
Thus, the performance gap cannot be attributed to an unfavorable accounting of context length or token usage.

Across models, we observe that the relative disadvantage of agents is most pronounced under tight budgets, while the gap narrows only marginally as the budget increases.
This suggests that, under realistic deployment constraints, allocating inference resources to multiple independent samples is substantially more effective than investing them into deeper agent trajectories.
These findings complement prior observations in multi-agent settings, where much of the empirical benefit arises from aggregation and diversity rather than prolonged deliberation \citep{debateOrVote,madSparseTopology}.

Overall, the results demonstrate that for competitive programming, \kshot inference achieves a markedly superior accuracy--cost and accuracy--query tradeoff compared to agent-based approaches, motivating a closer examination of agent failure modes in the following section.




\section{Agent Trajectory Analysis}
Figures~\ref{fig:cum-solved-vs-cost} and~\ref{fig:cum-solved-vs-queries-o3} demonstrate that, when performance is normalized by either monetary cost or number of model calls, \kshot inference, based on repeated independent model invocations, consistently outperforms both agent-based methods and \agentk{} across all evaluated settings.

At first glance, this finding may appear counterintuitive. Agent-based approaches have shown strong performance in complex decision-making and interactive reasoning tasks, where iterative refinement, tool use, and feedback are critical \citep{yao2023react,li2023theoryofmind}. One might therefore expect such methods to dominate simpler repeated-sampling strategies in competitive programming as well.

To reconcile this discrepancy, we conduct a detailed analysis of agent trajectories to identify systematic failure modes. In particular, we use an LLM-based judge, GPT-4.1, to categorize and assess agent behaviors across unsuccessful runs.

Our analysis reveals several recurring weaknesses in agent-based reasoning, summarized in Table~\ref{tab:weakness-summary}. A more detailed description of these weaknesses, including examples, can be found in Section~\ref{sec:agent-weakness-detailed}.
\begin{table}[h]
\centering
\small
\begin{tabular}{|p{0.25\linewidth}|c|p{0.55\linewidth}|}
\hline
\textbf{Weakness Category} & \textbf{Frequency} & \textbf{Description} \\
\hline
Algorithmic Inefficiency \& Time Complexity Issues & 19.3\% &
Failure to design efficient algorithms, leading to brute-force, high-polynomial, or exponential solutions where more efficient approaches are required. \\
\hline
Iterative Loop Without Breakthrough & 7.4\% &
Repeated cycling through similar incorrect strategies with incremental tweaks, without reconsidering the core approach. \\
\hline
Ineffective Debugging \& Problem-Solving Strategy & 7.0\% &
Excessive focus on peripheral issues and ineffective use of feedback, rather than structured reasoning and strategic reassessment. \\
\hline
Missing Mathematical / Theoretical Insight & 6.7\% &
Failure to identify key mathematical properties, combinatorial insights, or theoretical foundations needed for efficient solutions. \\
\hline
Fundamental Problem Misunderstanding & 0.9\% &
Misinterpretation of core problem constraints or mechanics, resulting in solutions that do not address the actual task. \\
\hline
\end{tabular}
\caption{Primary categories of reasoning and problem-solving failures observed in agents.}
\label{tab:weakness-summary}
\end{table}



\section{Scaling Solvers Under Fixed Budgets}\label{sec:broader-discussion}

Consider trying to solve a problem using $n$ solvers, where running the $i$th solver costs $c^{(i)} > 0$ (dollars) and has success probability $p^{(i)} \in (0, 1)$. We can run each solver as many times as we wish independently. Given a fixed budget $C$ (dollars), our goal is to maximize the final probability of success.

If we chose to run the first solver $k^{(1)}$ times, the second solver $k^{(2)}$ times and so on, the failure probability would be
\begin{equation*}
(1 - p^{(1)})^{k^{(1)}} \ldots (1 - p^{(n)})^{k^{(n)}}
\end{equation*}
and the cost would be
\begin{equation*}
k^{(1)} c^{(1)} + \ldots + k^{(n)} c^{(n)}.
\end{equation*}
So the goal becomes
\begin{equation}\label{eq:dYSpYm}
\begin{aligned}
  \text{maximize}& \quad - k^{(1)} \ln (1 - p^{(1)}) - \ldots - k^{(n)} \ln (1 - p^{(n)})\\
  \text{s.\,t.}& \quad k^{(1)} c^{(1)} + \ldots + k^{(n)} c^{(n)} \leq C,\\
  &\quad k^{(1)}, \ldots, k^{(n)} \in \mathbb{Z}_{\geq 0}.
\end{aligned}
\end{equation}
This is an unbounded knapsack problem \citep{kellerer2004knapsack}. When $C$ is large compared to each $c^{(i)}$, we can use a linear programming relaxation of this problem where $k^{(i)}$ do not have to be integers.

\begin{proposition}
The solution of the LP-relaxation to \eqref{eq:dYSpYm},
\begin{equation}\label{eq:ebYNit}
\begin{aligned}
  \text{maximize}& \quad - k^{(1)} \ln (1 - p^{(1)}) - \ldots - k^{(n)} \ln (1 - p^{(n)})\\
  \text{s.\,t.}& \quad k^{(1)} c^{(1)} + \ldots + k^{(n)} c^{(n)} \leq C,\\
                 &\quad k^{(1)}, \ldots, k^{(n)} \geq 0,
\end{aligned}
\end{equation}
is given by
\begin{equation*}
k^{(j)} = \frac{C}{c^{(j)}}, \quad k^{(1)} = \ldots = k^{(j - 1)} = k^{(j + 1)} = \ldots = k^{(n)} = 0,
\end{equation*}
where
\begin{equation*}
j \in \argmax_i \frac{- \ln (1 - p^{(i)})}{c^{(i)}}.
\end{equation*}
\end{proposition}

\begin{proof}
This is well-known, see e.\,g. Lemma 8.1.1 in \citet{kellerer2004knapsack}.
By definition of $j$, we have
\begin{align*}
  &- \bar{k}^{(1)} \ln (1 - p^{(1)}) - \ldots - \bar{k}^{(n)} \ln (1 - p^{(n)})\\
  &\quad = \bar{k}^{(1)} c^{(1)} \frac{- \ln (1 - p^{(1)})}{c^{(1)}} + \ldots + \bar{k}^{(n)} c^{(n)} \frac{- \ln (1 - p^{(n)})}{c^{(n)}}\\
  &\quad \leq \bar{k}^{(1)} c^{(1)} \frac{- \ln (1 - p^{(j)})}{c^{(j)}} + \ldots + \bar{k}^{(n)} c^{(n)} \frac{- \ln (1 - p^{(j)})}{c^{(j)}}\\
  &\quad = (\bar{k}^{(1)} c^{(1)} + \ldots + \bar{k}^{(n)} c^{(n)}) \frac{- \ln (1 - p^{(j)})}{c^{(j)}}\\
  &\quad \leq C \frac{- \ln (1 - p^{(j)})}{c^{(j)}} = k^{(j)} \brk{- \ln (1 - p^{(j)})}
\end{align*}
for any $\crl{\bar{k}^{(i)}}_{i = 1}^n$ in the feasibility set.
\end{proof}

This motivates the following definition.

\begin{roundnote}
\begin{definition}
  For a method with success probability $p$ and attempt cost $c$, define its \textit{negative log failure likelihood per dollar} by
\begin{equation*}
\frac{- \ln (1 - p)}{c}.
\end{equation*}
\end{definition}
\end{roundnote}

Thus, this quantity is what matters for choosing which methods to apply for maximizing the success probability under a fixed budget.

Further, a solver can consist of, say, $q$ queries to an agent or some number of independent runs of the model. The following \lcnamecref{prop:cost_tradeoff} develops the scaling law that informs when to choose one versus the other.

\begin{proposition}
Assume one query of the model has success probability $p_1$ and cost $c_1$, while an agent using $q \geq 2$ API calls
has success probability $\pi_q$ and cost $\zeta_q$ that has a power-law dependence on $q$: $\zeta_q = a + b q^{\gamma}$, where $\gamma > 1$, and $a, b \in \mathbb{R}$. The agent's negative log failure likelihood per dollar is higher than that of  $k$ independent runs if and only if
\begin{equation}\label{eq:xOvrog}
\begin{aligned}
  &\frac{- \ln (1 - \pi_q)}{- \ln (1 - p_1)} \geq a' + b' q^{\gamma}, \quad\text{equivalently} \quad \pi_q \geq 1 - (1 - p_1)^{a' + b' q^{\gamma}},\\
  &\text{where} \quad a' := \frac{a}{c_1}, \quad b' := \frac{b}{c_1},
\end{aligned}
\end{equation}
regardless of $k$.
\label{prop:cost_tradeoff}
\end{proposition}

\begin{proof}
The agent's negative log failure likelihood per dollar is $- \ln (1 - \pi_q) / \zeta_q$ by definition. Running the model $k$ times independently costs $k c_1$ and has success probability
\begin{equation*}
1 - (1 - p_1)^k,
\end{equation*}
so the corresponding negative log failure likelihood per dollar is
\begin{equation*}
\frac{- k \ln (1 - p_1)}{k c_1} = \frac{- \ln (1 - p_1)}{c_1},
\end{equation*}
not depending on $k$.
Rewriting the condition
\begin{equation*}
\frac{- \ln (1 - \pi_q)}{a + b q^{\gamma}} \geq \frac{- \ln (1 - p_1)}{c_1}
\end{equation*}
yields \eqref{eq:xOvrog}.
\end{proof}

Proposition~\ref{prop:cost_tradeoff} formalizes the intuition that, when the cost outraces the agent's negative log failure likelihood per dollar, it is more cost-efficient to use \kshot.

In practice, the fast success probability growth condition \eqref{eq:xOvrog} appears to be difficult to satisfy.
For small $q$, the agent's cost $\zeta_q$ is smaller than the corresponding $q$-shot cost $q c_1$ because the model does the easy work of figuring out its environment
(such as current directory, file contents and so on), as opposed to solving a problem in one query.
However, the cost scaling exponent $\gamma$ can be much higher than $1$ because of context growth.

\paragraph{Scaling trends in practice.}
\Cref{fig:scaling} shows the
what the scaling trends look like in reality. We choose a sufficiently easy problem in Division 3, which Claude Sonnet 4 (with final softmax temperature 0.5) solves in one attempt with probability $p_1$ estimated at just under $30\%$. This means that the success probability of $k$-shot sampling $1 - (1 - p_1)^k$ converges to one exponentially, with log failure likelihood $k \ln (1 - p_1)$ diverging to $- \infty$ linearly, and negative log failure likelihood per dollar constant:

\begin{equation*}
\frac{k \ln (1 - p_1)}{k c_1} = \frac{\ln (1 - p_1)}{c_1}.
\end{equation*}
We see that the success likelihood of an agent (with the same model and temperature) converges to one more slowly when matching the cost, and correspondingly the agent's negative log failure likelihood per dollar
is much lower for agents.
Interestingly, the agent's success likelihood scales near-linearly in the budget, whereas the cost
grows super-linearly in the number of queries.
All calculated probabilities are estimated with Clopper-Pearson 95\% confidence intervals of length 0.1.

\begin{figure}[h!]
\centering
\begin{tabular}{cc}
\begin{subfigure}[t]{0.48\textwidth}
\includegraphics[width=\linewidth]{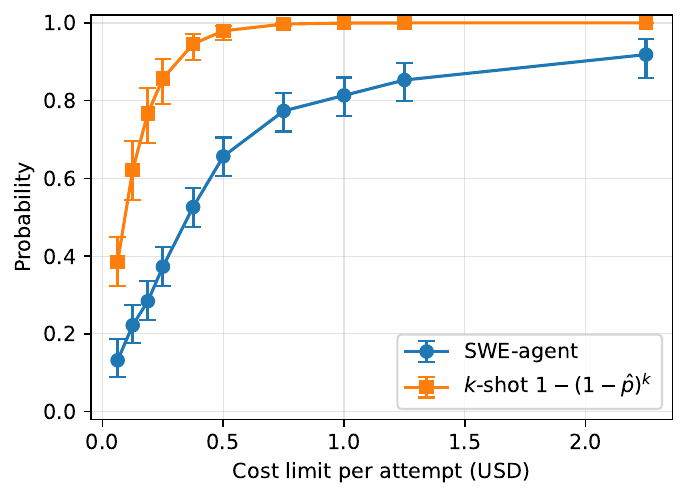}
\caption{The scaling curve of the success probability}
\end{subfigure}
&
\begin{subfigure}[t]{0.48\textwidth}
\includegraphics[width=\linewidth]{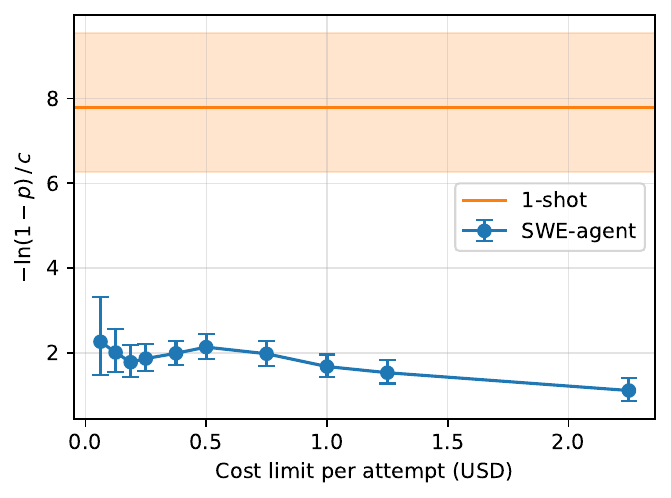}
\caption{The negative log failure likelihood per dollar is much lower for agents compared to a one-shot attempt\label{fig:neg-log-fail-likel-per-dollar}}
\end{subfigure}
\\
\begin{subfigure}[t]{0.48\textwidth}
\includegraphics[width=\linewidth]{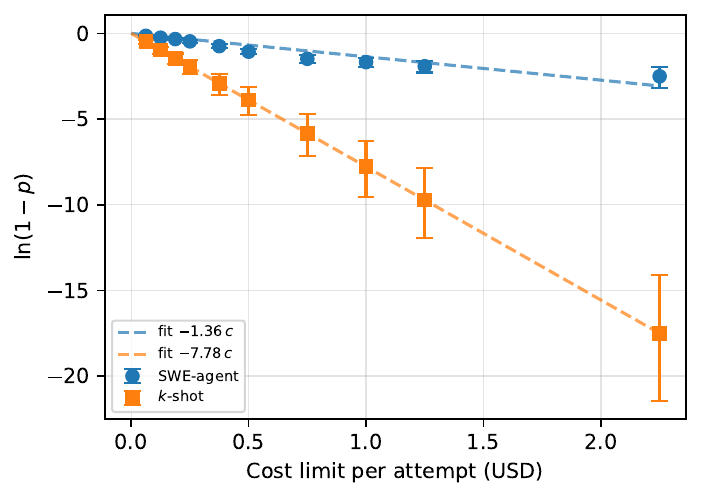}
\caption{The log failure likelihood as a function of the budget scales near-linearly for agents as well $k$-shot attempts}
\end{subfigure}
&
\begin{subfigure}[t]{0.48\textwidth}
\includegraphics[width=\linewidth]{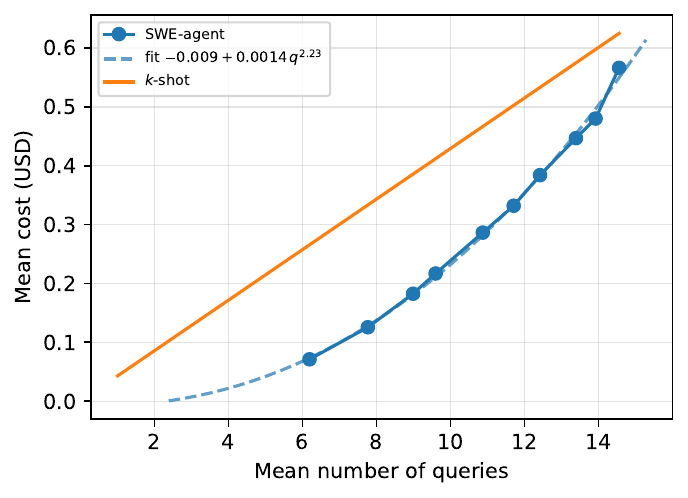}
\caption{The best power law fit of the agent's cost given $q$ queries is near-quadratic in $q$}
\end{subfigure}
\end{tabular}
\caption{Scaling trends of $k$-shot attempts vs. agents, on a Division 3 problem. We see that SWE-agent is less cost-efficient by our metric. Interestingly, the log-failure probability is linear in the cost limit for the agents (as well independent attempts). See \cref{sec:additional-results} for a similar plot with a much harder problem.}
\label{fig:scaling}
\end{figure}

\paragraph{Broader adaptability of scaling law in \cref{prop:cost_tradeoff}.} More broadly, this discussion motivates cutting short
any stateful search with $q \geq 2$ moves, success probability $\pi_q$ and cost $\zeta_q$ (e.\,g. an agent with $q$ API queries)
for cost efficiency.
Specifically,
the value $q^{\star}$ maximizing
the negative log failure likelihood per dollar
\begin{equation*}
q^{\star} = \argmax_{q \geq 1} \frac{- \ln (1 - \pi_q)}{\zeta_q}
\end{equation*}
(assume uniqueness for simplicity)
can be considered the cost-optimal (analogously to compute-optimal) number of moves. If such a number is possible to find, one should not give the search more than $q^{\star}$ moves; instead, independent searches with exactly $q^{\star}$ moves each should be repeated.

In particular, with less reasoning-intensive and more interactive tasks (where agents are strictly necessary), this framework helps decide the cost-optimal amount of resources one agent replica should be allocated.





\section{Related Work}\label{sec:related-work}


\paragraph{Multi-agent collaboration and debate.}
A growing body of work studies multi-agent collaboration as a means of improving the reasoning capabilities of large language models.
Multi-Agent Debate (MAD) proposes that allowing agents to propose, critique, and refine solutions can elicit more robust reasoning and reduce individual errors \citep{MAD}.
Subsequent empirical analyses, however, suggest that much of the observed gain in MAD-style systems arises from aggregation rather than sustained deliberation.
In particular, \citet{debateOrVote} demonstrate that majority voting over independent agent outputs accounts for most of the performance improvements attributed to debate.
This observation closely aligns with our findings, where repeated independent sampling (\kshot) consistently outperforms agent-based approaches under fixed inference budgets.

Further work has examined the role of communication structure in multi-agent systems.
\citet{madSparseTopology} show that restricting communication to sparse topologies can outperform fully connected debate graphs, indicating that excessive interaction may be counterproductive.
Together, these results suggest that unconstrained agent interaction can introduce inefficiencies, especially when inference-time resources are limited.

\paragraph{Agentic reasoning and tool use.}
Beyond debate-style collaboration, agentic reasoning frameworks emphasize interleaving reasoning with actions such as tool invocation and environment interaction.
ReAct \citep{yao2023react} formalizes this paradigm and demonstrates strong gains on tasks that require iterative decision-making.
Related work on theory-of-mind modeling further highlights how agents can reason about other agents’ beliefs and intentions in collaborative settings \citep{li2023theoryofmind}.
Agent-based systems have been particularly successful in software engineering benchmarks, such as SWE-bench, where agents must navigate large codebases, edit files, and run tests \citep{jimenez2024swebench,NEURIPS2024_5a7c9475}.

In contrast, competitive programming problems are typically self-contained and emphasize algorithmic insight rather than environment exploration.
Our results show that, under realistic inference budgets, the overhead associated with iterative agent trajectories often outweighs their benefits in this domain.


\paragraph{Accuracy-cost tradeoffs in agent systems.}
A growing line of work argues that agent evaluations should account for inference cost rather than accuracy alone.
Most closely related, \citet{kapoor2025ai} argue that agent evaluations must be cost-controlled
and advocate visualizing cost-accuracy tradeoffs using Pareto frontiers.
Our empirical results complement this perspective on a different class of programming tasks.
We further refine their theoretical perspective
by specializing the cost-aware view to a particular deployment objective: maximizing final success probability under a fixed budget when solvers can be retried independently.
For this objective, we prove rigorously that the relevant scalar is the reliability exponent per dollar $-\ln(1-p)/c$, thus capturing accuracy and cost in a principled way.
In addition, we address their concerns of lacking error bars in agent evaluations by including
exact confidence intervals in our scaling figures (\cref{fig:scaling}).
\Citet{kim2025towards} account for costs by considering ``success per 1000 tokens'' and cost-normalizing performance metrics, guided by intuitions we take further by assuming a system reliability perspective. Hence, we provide a theoretically principled complement to the scaling science. Similarly, the work \citep{erol2026costofpass} proposes a different though related metric ``cost of pass'' which would effectively be $c / p$ in the setting of \eqref{eq:GNnqtn}. This metric and ours agree to first order when $p$ is small (since $- \ln(1 - p) \approx p$), but differ for higher-success solvers.

\paragraph{Budget-aware inference-time scaling.}
Our work is also related to budget-aware evaluations of inference-time reasoning and repeated sampling.
\Citet{wang2024reasoningtoken} evaluate reasoning strategies under matched query, token, and monetary budgets, finding that simple self-consistency baselines can outperform more complex strategies once compute is controlled, and that some strategies such as multi-agent debate or Reflexion can become worse when compute is fixed.
\Citet{brown2024large} study repeated sampling as an inference-time scaling method, showing that log-coverage scales at a power law with the number of generated samples, especially in domains such as coding and formal proofs where candidates can be automatically verified.
\Citet{zhang2025scaling} propose OSCA, which learns mixed allocations of sample compute across inference configurations and can outperform the best single configuration under benchmark-level compute budgets.
Their benchmark-level pass@$C$ objective averages the success probability
over problem instances, where different configurations can specialize on different instances.
This heterogeneity can make mixed allocations optimal, even in the continuous relaxation.
In contrast, our reliability analysis is for a fixed success-or-failure problem with fixed per-attempt probabilities and costs in dollars, allowing for a simple metric which incorporates both accuracy and real cost.

\paragraph{Competitive programming with LLMs.}
Several recent studies have directly examined the performance of LLMs on competitive programming.
Zheng et al.~\citep{zheng2025livecodebenchproolympiadmedalists} evaluate LLMs on difficult contest problems and observe a substantial gap between benchmark performance and results reported in model announcements.
Other works evaluating reasoning models on competitive programming problems given one or multiple independent attempts include \citet{quan2025codeelobenchmarkingcompetition,li2025humanityslast,wei2025evaluatingimprovinglarge}.
OpenAI et al.~\citep{openai2025competitiveprogramminglargereasoning} demonstrate that reasoning-focused models with terminal access can discover sophisticated test-time strategies, such as using brute-force solvers for validation.
\Citet{yang2025elaborationcomprehensivebenchmark} use LLM-based simulators of human feedback.
\Citet{xu2025icpceval} evaluate reasoning models on ICPC problems, and introduce a metric Refine@K which allows the model to iteratively refine its solution within a budget of $k$ attempts.
While these results highlight the potential of agentic interaction, they largely assume generous inference budgets.
Our work complements these studies by explicitly focusing on the cost--accuracy tradeoff and showing that simpler inference strategies can dominate under strict budget constraints.






\section{Conclusion and Future Directions}

We investigated how to allocate inference-time compute for competitive programming under fixed budget constraints.
By evaluating performance as a function of both monetary cost and number of model calls, we showed that repeated independent sampling (\kshot) consistently outperforms agent-based approaches across models and difficulty levels.

This gap is not explained by accounting effects such as longer contexts or higher per-call costs.
Rather, it reflects a mismatch between agentic iteration and the structure of competitive programming tasks: agents often expend substantial budget on unproductive refinement, while \kshot benefits from broad exploration that surfaces correct solutions early.

Our findings do not diminish the value of agent-based methods in interactive domains such as software engineering. Instead, they highlight that inference strategies should be matched to task structure and resource constraints. For self-contained algorithmic problems, simple exploration via independent sampling provides a superior accuracy--cost tradeoff.

A limitation of this study is that it evaluates only a single agent framework, SWE-agent. Although we customize its prompts and scaffolding, the observed gap may still underestimate the potential of agentic methods that are more deeply tailored to competitive programming.
It is an exciting future direction to test our setting with multi-agent collaboration.

We hope our work encourages cost-aware evaluation of inference strategies and informs the design of efficient reasoning systems under realistic deployment constraints.

\section*{Acknowledgments}

We extend our very special gratitude to Mike Mirzayanov, the creator and maintainer of Codeforces, who
kindly provided the opportunity to run evaluations through a private research API.
This project would not have happened without him.
We also very much thank Kilian Lieret for important discussions, including the ones on the framing, as well as contributions towards the custom SWE-agent fork and \kshot implementation. 



\bibliographystyle{plainnat}
\bibliography{infinite}


\newpage
\appendix

\begin{internal}
\section{INTERNAL experiment notes}

\subsection{Estimating probabilities}

I take Contest 2060 (Div. 3), problem B, run Sonnet a few hundred times to estimate $p_1$ with the $95\%$ confidence interval of length about $0.1$. The script is \Verb{exps/launch_probability_estimation.py}.

I run agents with different cost limits and also estimate the probability until the Clopper-Pearson exact confidence interval has length 0.1. The script is \Verb{exps/launch_agent_probability_estimation.py}.
The logs are in the \Verb{LOG_DIR} in this file, specifically \Verb{agent_trajectories/logs}.
The actual trajectories are saved in \Verb{trajectories/borshigida} of the SWE-agent path as usual.

\begin{verbatim}
python3 -c 'from exps.launch_probability_estimation import main; main(div=3, contest_id="2060", problem_index="B")'
python3 -c 'from exps.launch_agent_probability_estimation import main; main(div=3, instance_id="2060B", cost_limit=0.0625)'
python3 -c 'from exps.launch_agent_probability_estimation import main; main(div=3, instance_id="2060B", cost_limit=0.125)'
python3 -c 'from exps.launch_agent_probability_estimation import main; main(div=3, instance_id="2060B", cost_limit=0.1875)'
python3 -c 'from exps.launch_agent_probability_estimation import main; main(div=3, instance_id="2060B", cost_limit=0.25)'
python3 -c 'from exps.launch_agent_probability_estimation import main; main(div=3, instance_id="2060B", cost_limit=0.375)'
python3 -c 'from exps.launch_agent_probability_estimation import main; main(div=3, instance_id="2060B", cost_limit=0.50)'
python3 -c 'from exps.launch_agent_probability_estimation import main; main(div=3, instance_id="2060B", cost_limit=0.75)'
python3 -c 'from exps.launch_agent_probability_estimation import main; main(div=3, instance_id="2060B", cost_limit=1.00)'
python3 -c 'from exps.launch_agent_probability_estimation import main; main(div=3, instance_id="2060B", cost_limit=1.25)'
python3 -c 'from exps.launch_agent_probability_estimation import main; main(div=3, instance_id="2060B", cost_limit=2.25)'
\end{verbatim}

Harder problem:
\begin{verbatim}
python3 -c 'from exps.launch_probability_estimation import main; main(div=1, contest_id="2115", problem_index="A")'
python3 -c 'from exps.launch_agent_probability_estimation import main; main(div=1, instance_id="2115A", cost_limit=0.125)'
python3 -c 'from exps.launch_agent_probability_estimation import main; main(div=1, instance_id="2115A", cost_limit=0.1875)'
python3 -c 'from exps.launch_agent_probability_estimation import main; main(div=1, instance_id="2115A", cost_limit=0.75)'
python3 -c 'from exps.launch_agent_probability_estimation import main; main(div=3, instance_id="2060B", cost_limit=1.25)'
python3 -c 'from exps.launch_agent_probability_estimation import main; main(div=3, instance_id="2060B", cost_limit=2.25)'
\end{verbatim}
The $c = 0.125$ runs have success rate zero; for $c = 0.1875$ it is $0.0185$.
\end{internal}

\section{Main Agent Weaknesses}
\label{sec:agent-weakness-detailed}
\subsection*{1. Algorithmic Inefficiency \& Time Complexity Issues (19.3\%)}

Agents fail to develop efficient algorithms, resulting in Time Limit Exceeded errors.
They are often stuck using brute-force or exponential approaches when
polynomial or logarithmic solutions are required.
\internalComment{Is this a quote? TODO otherwise need some explanation. <--Updated.}

\begin{example}[Problem 2097B, Claude, Div.~1]
The agent relied on brute-force/backtracking approaches that are inherently exponential
and unsuitable for large constraints, failing to recognize the need for a fundamentally
different (likely combinatorial or DP) approach early enough.
\end{example}

\begin{example}[Problem 2101C, O3, Div.~1]
There was significant code churn and repeated attempts at similar greedy strategies,
often with subtle logic errors (e.g., incorrect pairing, off-by-one mistakes, or improper
feasibility checks), leading to wasted computational resources.
\end{example}

\subsection*{2. Iterative Loop Without Breakthrough (7.4\%)}

The agent cycles between similar incorrect approaches without making fundamental changes,
making incremental tweaks rather than reconsidering the core strategy.

\begin{example}[Problem 2072D, Div.~3]
It spent many steps cycling through similar $O(n^4)$ or $O(n^3)$ approaches without
successfully breaking through to a fundamentally more efficient method (e.g., using
advanced data structures or inversion delta analysis).
\end{example}

\begin{example}[Example 2 (Problem 2104B, Div.~2]
The agent fundamentally misunderstood the nature of the required optimization,
repeatedly reverting to incorrect heuristics that did not match the problem's requirements.
\end{example}

\subsection*{3. Ineffective Debugging \& Problem-Solving Strategy (7.0\%)}

The agent spends excessive time on peripheral issues, does not step back to reconsider
fundamentals, or fails to leverage feedback from test failures effectively.

\begin{example}[Problem 2104G, Div.~2]
There is a lack of clear, structured reasoning about the relationship between the graph's
structure and the coloring operation, resulting in inefficient trial-and-error rather than
principled problem solving.
\end{example}

\begin{example}[Problem 2077F, Div.~1]
The agent spent excessive effort on peripheral analyses (sum, xor, sorting, increment-only)
rather than focusing on the specific bitwise operation mechanics required by the problem.
\end{example}

\subsection*{4. Missing Mathematical/Theoretical Insight (6.7\%)}

The agent fails to identify or derive the key mathematical property, combinatorial insight,
or theoretical foundation needed for an efficient solution.

\begin{example}[Problem 2115F2, Div.~1]
It did not make the key insight needed for this problem: to use a data structure or algorithm
(such as segment trees, persistent data structures, or lazy propagation) that would allow
efficient range operations and queries.
\end{example}

\begin{example}[Problem 2114F, Div.~3]
The mathematical insight into the problem was only partially developed; the agent did not
fully optimize the factorization approach for performance, missing opportunities for further
optimization (e.g., precomputing primes, using efficient exponent counting).
\end{example}

\subsection*{5. Fundamental Problem Misunderstanding (0.9\%)}

The agent fundamentally misunderstands the problem's core constraints, requirements, or
mechanics, leading to solutions that do not address what is actually being asked.

\begin{example}[Problem 2096A, Div.~1]
The agent fundamentally misunderstood the problem's constraints, especially the meaning
of `\textless{}' and `\textgreater{}' in the context of ``smaller/larger than all previous
sticks,'' leading to repeated implementation of incorrect algorithms.
\end{example}

\begin{example}[Problem 2117E, Div.~3]
Fundamental misunderstanding of the problem's propagation mechanics led to repeated
incorrect implementations.
\end{example}

\section{Prompts}\label{sec:prompts}

\begin{tcolorbox}[
  title={API calls prompt},
  colback=white,
  colframe=black!50,
  fonttitle=\bfseries,
  boxrule=0.6pt,
  left=6pt,right=6pt,top=6pt,bottom=6pt,
  enhanced,
  breakable,
  ]

\textit{You are an expert at competitive coding challenges.}

Please solve the following problem

\begin{verbatim}
<problem_statement>
{{ problem_statement }}
</problem_statement>
\end{verbatim}

Here is some template for your solution

\begin{verbatim}
<template>
#include <bits/stdc++.h>

using namespace std;

#define ve vector
#define pa pair
#define tu tuple
#define mp make_pair
#define pb push_back
#define sz(a) ((int)(a).size())

#define ARG4(_1,_2,_3,_4,...) _4

#define forn3(i,l,r) for (int i = int(l); i < int(r); ++i)
#define forn2(i,n) forn3 (i, 0, n)
#define forn(...) ARG4(__VA_ARGS__, forn3, forn2) (__VA_ARGS__)

#define ford3(i,l,r) for (int i = int(r) - 1; i >= int(l); --i)
#define ford2(i,n) ford3 (i, 0, n)
#define ford(...) ARG4(__VA_ARGS__, ford3, ford2) (__VA_ARGS__)

const long double eps = 1e-9;
const int inf = (1 << 30) - 1;
const long long inf64 = (1LL << 62) - 1;

template<typename T> inline T abs (T x) {return x < 0 ? -x : x;}

double EPS = 1e-9;
long long INF = LLONG_MAX;

// ...

int main() {
    ios::sync_with_stdio(0);
    cin.tie(0);

#ifdef LOCAL
    freopen("input.txt", "r", stdin);
#endif

    // ...

    return 0;
}
</template>
\end{verbatim}

Your answer should include be formatted as follows

\begin{verbatim}
Thinking process

<code>
Your C++ code here
</code>
\end{verbatim}
\end{tcolorbox}

\begin{tcolorbox}[
  title={SWE-agent prompt},
  colback=white,
  colframe=black!50,
  fonttitle=\bfseries,
  boxrule=0.6pt,
  left=6pt,right=6pt,top=6pt,bottom=6pt,
  enhanced,
  breakable,
  ]

\textit{You are a competitive programmer solving a CodeForces contest.}

We're currently solving the following coding challenge:
\begin{verbatim}
{{problem_statement}}
\end{verbatim}

\medskip

INSTRUCTIONS:

\medskip

Your solution should be added to the directory [...].
Currently, the only file in this directory is [...] with a template for your solution.
However, you can also create additional files in this directory.

\medskip

At every step, please include a SINGLE function call.

\medskip

There are several helpful tools available, but you're free to use any other bash commands you want (e.g. find, grep, cat, ls, cd) in addition to the special commands listed above.
However, the environment does NOT support interactive session commands (e.g. python, vim), so please do not invoke them.

\medskip

Now, you're going to solve this problem on your own. Your terminal session has started and you're in the directory [...]. You can use any bash commands or the special interface to help you.

\medskip

IMPORTANT TIPS:
\begin{enumerate}
\item[1.] Use C++ for your solution. Use \verb|g++ -O2 -std=c++17 -DLOCAL -o res.out| for compilation, and \verb|./res.out| to execute. There is no need to specify the input file (it is already handled by the freopen call in the template).

\item[1a.] FOR INPUT, USE \verb|input.txt|! DO NOT USE STDIN FOR TESTING YOUR PROGRAMS.

\item[2.] For efficiency, chain your compilation and execution commands (\verb|g++ ... && echo "..." > input.txt && ./res.out| etc.).

\item[2a.] If you're making smaller edits, you can also chain your edits: \verb|str_replace_editor ... && g++ ... && echo "..." "> input.txt && ./res.out|

\item[3.] Once you have a solution candidate, call the \verb|submit_to_cf| tool to submit your solution to CodeForces

\item[3a.] There's no need to test before you run \verb|submit_to_cf|

\item[3b.] However, if \verb|submit_to_cf| returns a verdict that is not OK, it might be very helpful to start creating additional test cases to understand failure mode.
\end{enumerate}

YOU CANNOT GIVE UP BEFORE YOU RECEIVED AN OK VERDICT! IF YOU ARE STUCK, STILL TRY SOMETHING (e.g. try special cases)!
\end{tcolorbox}

\section{Additional Results}\label{sec:additional-results}

We refer to \cref{fig:cum-solved-vs-cost,fig:cum-solved-vs-queries} for the full version
of \cref{fig:cum-solved-vs-cost-o3,fig:cum-solved-vs-queries-o3} (with two more models).
In addition, we plot in \cref{fig:scaling-2115A} the scaling trends on a much harder Division 1 problem than in \cref{fig:scaling}.

\py{fig_cum_solved_vs_cost()}
\begin{figure}[htb!]
\centering
\begin{tabular}{c}
\begin{subfigure}[t]{0.900\textwidth}
\includegraphics[width=\linewidth]{solved_vs_cost_o3.pdf}
\caption{OpenAI O3}
\end{subfigure}
\\
\begin{subfigure}[t]{0.900\textwidth}
\includegraphics[width=\linewidth]{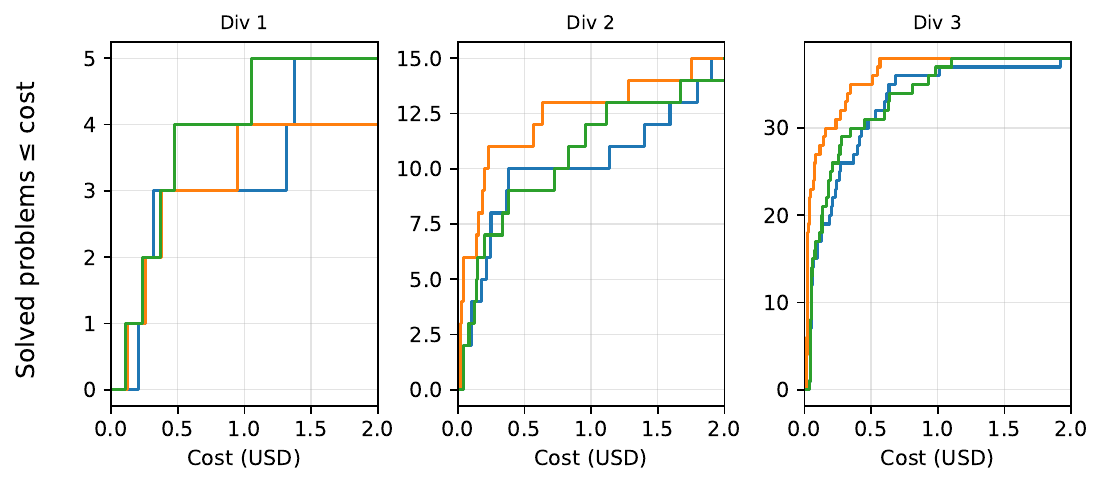}
\caption{Claude Sonnet 4}
\end{subfigure}
\\
\begin{subfigure}[t]{0.900\textwidth}
\includegraphics[width=\linewidth]{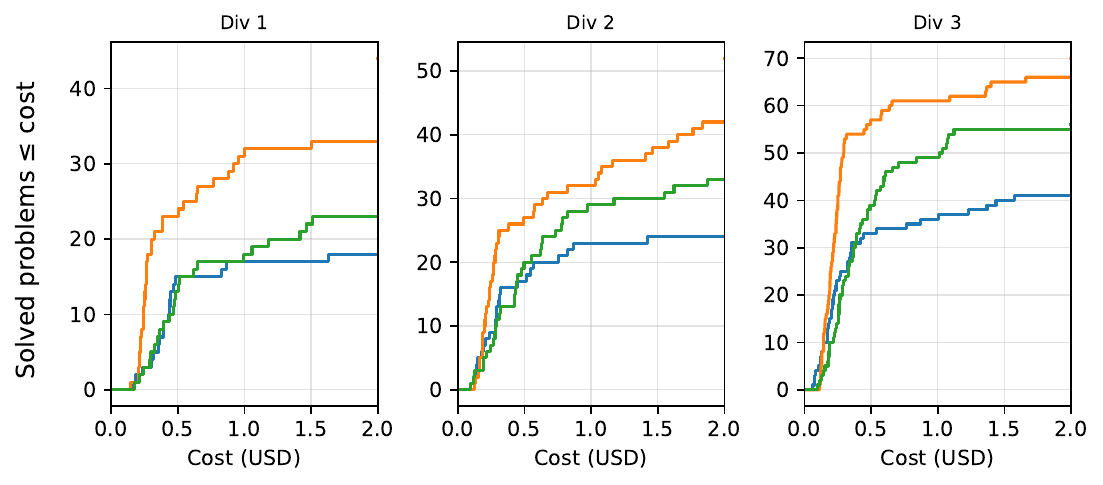}
\caption{Gemini 2.5 Pro}
\end{subfigure}
\end{tabular}
\caption{Cumulative solved problems versus inference cost: {\color{mplblue}agent}, {\color{mplorange}$k$-shot}, {\color{mplgreen}budget-partitioned agent ($\nicefrac{1}{3}\times 3$)}.}
\label{fig:cum-solved-vs-cost}
\end{figure}

\py{fig_cum_solved_vs_queries()}
\begin{figure}[htb!]
\centering
\begin{tabular}{c}
\begin{subfigure}[t]{0.900\textwidth}
\includegraphics[width=\linewidth]{solved_vs_queries_o3.pdf}
\caption{OpenAI O3}
\end{subfigure}
\\
\begin{subfigure}[t]{0.900\textwidth}
\includegraphics[width=\linewidth]{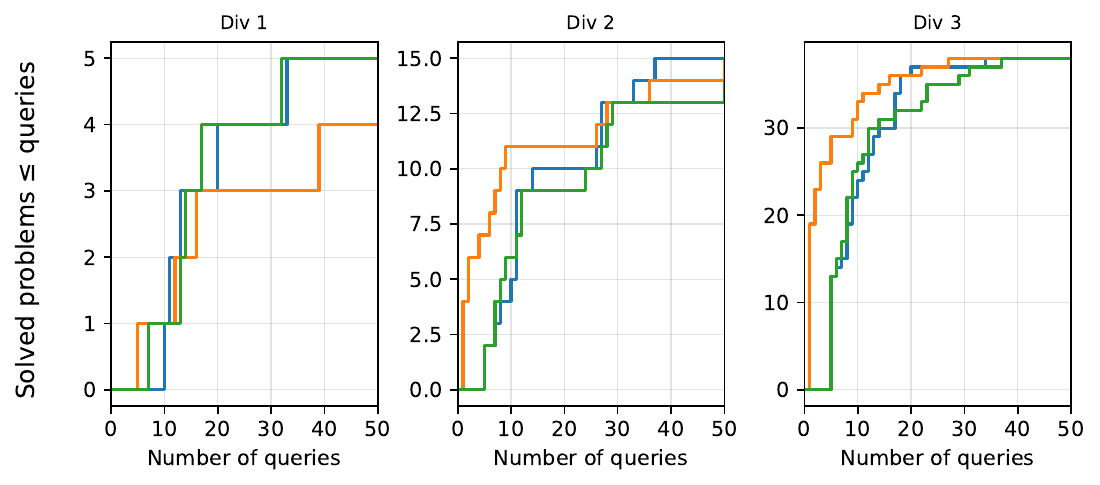}
\caption{Claude Sonnet 4}
\end{subfigure}
\\
\begin{subfigure}[t]{0.900\textwidth}
\includegraphics[width=\linewidth]{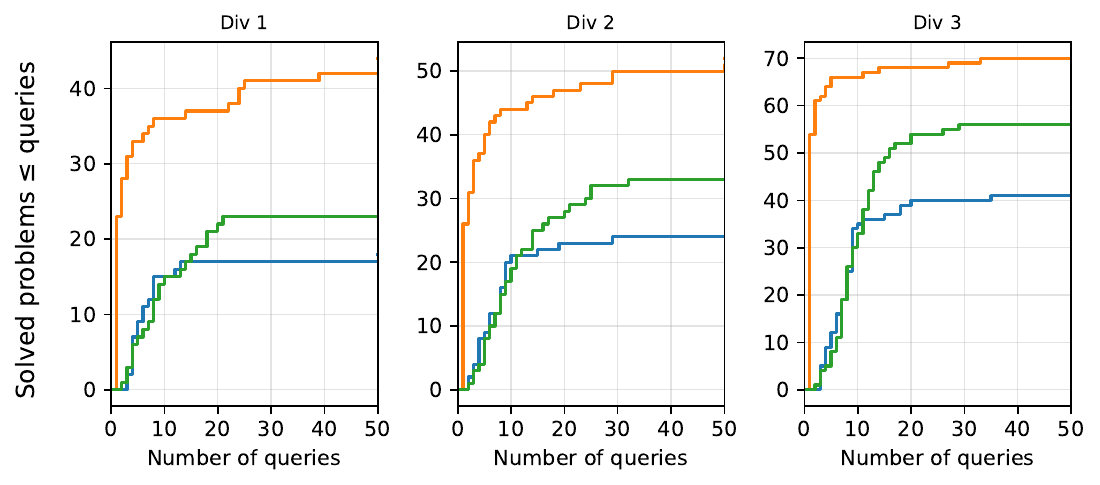}
\caption{Gemini 2.5 Pro}
\end{subfigure}
\end{tabular}
\caption{Cumulative solved problems versus number of queries: {\color{mplblue}agent}, {\color{mplorange}$k$-shot}, {\color{mplgreen}budget-partitioned agent ($\nicefrac{1}{3}\times 3$)}.}
\label{fig:cum-solved-vs-queries}
\end{figure}

\begin{figure}[h!]
\centering
\begin{tabular}{cc}
\begin{subfigure}[t]{0.48\textwidth}
\includegraphics[width=\linewidth]{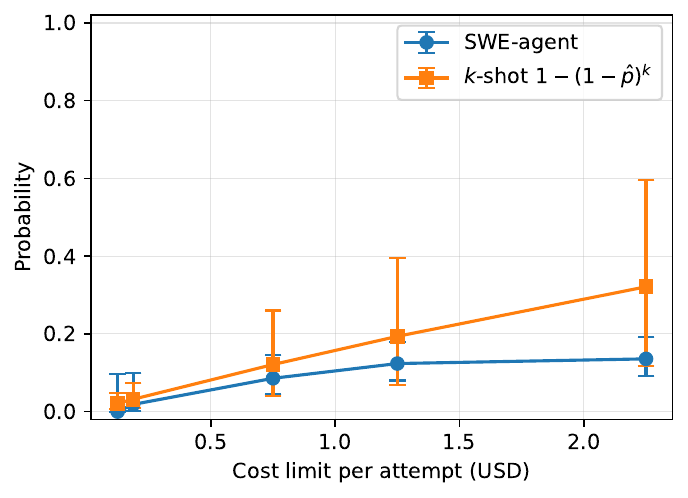}
\caption{The scaling curve of the success probability}
\end{subfigure}
&
\begin{subfigure}[t]{0.48\textwidth}
\includegraphics[width=\linewidth]{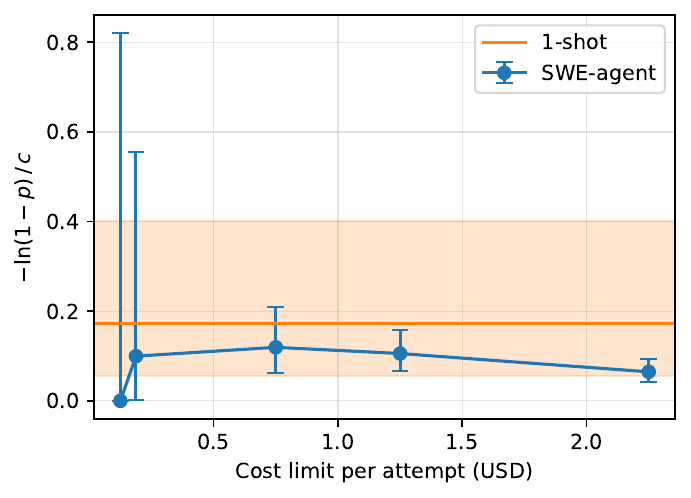}
\caption{The negative log failure likelihood per dollar is much lower for agents compared to a one-shot attempt\label{fig:neg-log-fail-likel-per-dollar-2115A}}
\end{subfigure}
\\
\begin{subfigure}[t]{0.48\textwidth}
\includegraphics[width=\linewidth]{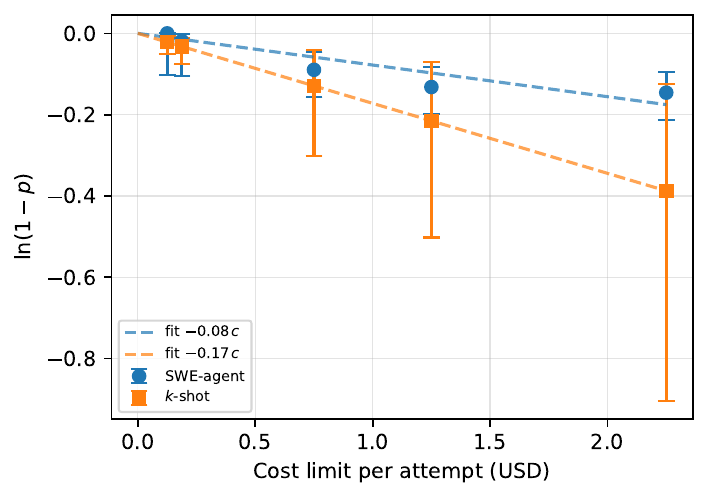}
\caption{The log failure likelihood as a function of the budget scales near-linearly for agents as well $k$-shot attempts}
\end{subfigure}
&
\begin{subfigure}[t]{0.48\textwidth}
\includegraphics[width=\linewidth]{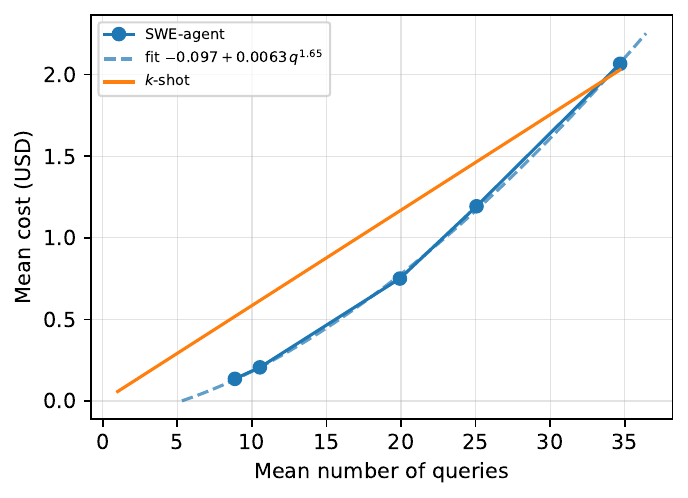}
\caption{The best power law fit of the agent's cost given $q$ queries is near-quadratic in $q$}
\end{subfigure}
\end{tabular}
\caption{Scaling trends of $k$-shot attempts vs. agents, on a Division 1 problem. The trends are similar to the ones in \cref{fig:scaling}. Because the problem is much harder, some confidence intervals are relatively wider and adding more points would incur higher costs.}
\label{fig:scaling-2115A}
\end{figure}


\end{document}